\documentclass[11pt]{article} 
\usepackage{algorithm,algorithmic}

\usepackage[utf8]{inputenc}
\usepackage[margin=1in]{geometry}
\usepackage{amsmath,amssymb,amsfonts,amsthm}
\usepackage{graphicx}
\usepackage{xcolor}
\usepackage[colorlinks=true,linkcolor=blue,citecolor=blue,urlcolor=blue]{hyperref}
\usepackage{natbib}

\newtheorem{theorem}{Theorem}[section]
\newtheorem{lemma}[theorem]{Lemma}

\newcommand{\acks}{\section*{Acknowledgments}}

\newcommand{\card}[1]{\left| #1 \right|}

\DeclareMathOperator*{\argmin}{arg\,min}

\title{Learning Augmented Graph $k$-Clustering}
\author{
  Chenglin Fan\thanks{Authors are listed in alphabetical order. This paper has been accepted for presentation at COLT 2025.},  Kijun Shin\footnotemark[1]\\
  Department of Computer Science and Engineering \\
  Seoul National University \\
{ \{fanchenglin, sharaelong\}@snu.ac.kr}
}
\date{}

\begin{document}

\maketitle

\begin{abstract}%
  Clustering is a fundamental task in unsupervised learning. Previous research has focused on learning-augmented $k$-means in Euclidean metrics, limiting its applicability to complex data representations. In this paper, we generalize learning-augmented $k$-clustering to operate on general metrics, enabling its application to graph-structured and non-Euclidean domains. Our framework also relaxes restrictive cluster size constraints, providing greater flexibility for datasets with imbalanced or unknown cluster distributions. Furthermore, we extend the hardness of query complexity to general metrics: under the Exponential Time Hypothesis (ETH), we show that any polynomial-time algorithm must perform approximately $\Omega(k / \alpha)$ queries to achieve a $(1 + \alpha)$-approximation. These contributions strengthen both the theoretical foundations and practical applicability of learning-augmented clustering, bridging gaps between traditional methods and real-world challenges.

\end{abstract}

\section{Introduction}

Clustering is a cornerstone of unsupervised machine learning, widely applied in fields such as data organization, anomaly detection, and community detection in networks~\citep{clustering_survey}.
Among clustering problems, the $k$-means and $k$-median problems stand out as fundamental due to their simplicity and effectiveness. Traditional algorithms aim to partition data into $k$ clusters, minimizing either the sum of squared distances (k-means) or the sum of absolute distances (k-median) to their respective cluster centers. The $k$-means algorithm has been a cornerstone of clustering research for decades, tracing its roots to foundational works by ~\citep{macqueen1967some} and ~\citep{lloyd1982least}, who introduced the iterative optimization approach still used today. Extensions by ~\citep{hartigan1979algorithm} improved convergence, while ~\citep{forgy1965cluster} proposed widely-used initialization techniques. The optimization principles underlying $k$-means were influenced by earlier algorithmic developments, such as Floyd's contributions to optimization~\citep{floyd1962algorithm}.  Improvements include $k$-means++~\citep{arthur2007kmeanspp}, which introduced a probabilistic seeding strategy to improve initialization quality and convergence, and Mini-Batch $k$-means\citep{sculley2010webscale}, which enabled clustering on massive datasets with reduced computational overhead. ~\citep{hamerly2010making} proposed optimizations to Lloyd's algorithm for faster convergence, while coreset-based approaches~\citep{lucic2016coresets, bachem2018scalable} have further accelerated clustering by providing compact, provably accurate data summaries.

Significant progress has been made in improving approximation algorithms for $k$-means and $k$-median clustering. For $k$-means, ~\citep{ahmadian2017primaldual} introduced a primal-dual algorithm achieving a 6.357-approximation, while their techniques also improved guarantees for Euclidean $k$-median. Later,~\citep{cohenaddad2022improved} further refined these results, achieving a 5.912-approximation for $k$-means and a 2.406-approximation for Euclidean $k$-median. For $k$-median, ~\citep{li2012pseudoapproximation} made a groundbreaking contribution with a $1 + \sqrt{3} + \epsilon$ ($\sim 2.732$-approximation) using a pseudo-approximation approach, relaxing the requirement to exactly $k$ clusters by allowing $k + O(1)$ clusters. These advances underscore the evolving theoretical landscape for clustering problems, yet they primarily address static Euclidean settings.

In general metric spaces, \citep{cohenaddad2020inapproximability} showed that approximating the continuous $k$-median within a factor better than 2 and the continuous $k$-means within a factor better than 4 is NP-hard, significantly improving upon prior inapproximability results. Further, under the Johnson Coverage Hypothesis, \citep{cohenaddad2021johnsoncoverage} demonstrated that discrete $k$-means is hard to approximate better than 3.94 in $\ell_1$ and 1.73 in $\ell_2$ metrics, underscoring the difficulty of achieving high-quality clustering solutions in geometric settings. In addition, lower bounds for coreset sizes, such as $\Omega(k \varepsilon^{-2} \log n)$ in finite metrics and $\Omega(k \varepsilon^{-2})$ in Euclidean spaces, were established by ~\citep{cohenaddad2022coreset}, revealing the trade-offs between coreset size and approximation quality.

These insights underscore the computational complexity and limitations of existing clustering approaches, motivating the need for new frameworks that address these challenges while expanding the applicability of clustering to dynamic and graph-structured data.

Traditional clustering methods, such as k-means, focus on partitioning data in Euclidean space by minimizing distances between data points and their assigned cluster centers. However, many real-world datasets are naturally represented as graphs, where relationships between data points are captured through edges and graph distances rather than explicit spatial coordinates. In such cases, conventional clustering methods face challenges, as Euclidean distance fails to capture the underlying graph structure~\citep{graph_clustering_review}. This motivates the need for algorithms that operate directly on graph-based data.

In parallel, the emergence of \emph{learning-augmented algorithms} has introduced a powerful framework for enhancing traditional methods by leveraging predictive models~\citep{learning_augmented_algorithms}. In the context of clustering, learning-augmented k-means algorithms incorporate learned information, such as approximate cluster centers, to improve computational efficiency and clustering quality~\citep{learning_augmented_kmeans}.
Similarly, \citep{gamlath2022noisylabels} addressed this challenge with the concept of Approximate Cluster Recovery from Noisy Labels, introducing an efficient algorithm capable of recovering clusters even when individual labels are perturbed with high probability (up to 99\%). This method achieves a clustering cost within a factor of $1 + o(1)$ of the optimum by leveraging side-information, even when significantly corrupted. 
Despite their potential, existing learning-augmented clustering methods have two key limitations:
\begin{enumerate}
    \item They rely exclusively on Euclidean distance metrics~\citep{learning_augmented_algorithms,gamlath2022noisylabels,learning_augmented_kmeans}, making them unsuitable for graph-structured data~\citep{euclidean_limitations}.
    \item They assume that each cluster must contain a minimum number of points~\citep{learning_augmented_algorithms,gamlath2022noisylabels}, limiting their applicability in scenarios where cluster sizes are highly imbalanced or unknown.
\end{enumerate}

This paper addresses these limitations simultaneouly by \emph{generalizing learning-augmented k-clustering} in three significant ways:
\begin{enumerate}
    \item \textbf{Extending Distance Metrics}: We replace Euclidean distances with graph-based distances, enabling the application of learning-augmented clustering to graph-structured data. This extension broadens the framework to new domains such as social networks, biological networks, and recommendation systems~\citep{graph_distances}.
    \item \textbf{Removing Cluster Size Constraints}: We develop a novel method that eliminates the need for minimum cluster size assumptions in general metrics. The previous findings concerning the removal of size limitations are exclusively applicable to the Euclidean distance, as cited in \citep{theoretical_guarantees}.
    This allows for greater flexibility in clustering settings where clusters naturally vary in size or where such constraints are impractical~\citep{yin2024rapid}.

   \textbf{ We propose a method to simultaneously achieve the two objectives mentioned above.}

    \item \textbf{ Hardness Results for Query Complexity}.  In this paper, we extend hardness results to query complexity for clustering in general metric spaces. Here, query complexity is defined such that the predictor assigns expected labels to input points, and retrieving the true label of a point constitutes one query—consistent with the framework in \cite{learning_augmented_kmeans}. Specifically, we show that under the Exponential Time Hypothesis (ETH), any polynomial-time algorithm must perform approximately $\Omega(k / \alpha)$ queries to output a $(1 + \alpha)$-approximate solution. This result generalizes the known hardness of query complexity from Euclidean distances~\citep{learning_augmented_kmeans} to general metrics, providing deeper insights into the computational barriers inherent in clustering problems across diverse settings.
\end{enumerate}

Our contributions are primarily theoretical, presenting a unified framework that incorporates graph distances into learning-augmented clustering while eliminating restrictive size constraints. We also offer a formal analysis of the algorithmic guarantees of our approach, demonstrating its robustness and generality.
By bridging the gap between graph clustering and learning-augmented methods, this work establishes a new foundation for adaptive clustering techniques capable of operating on diverse data structures. The proposed framework offers significant theoretical advancements while opening pathways for future research in learning-augmented algorithms for graph data.

\subsection*{High-level Idea}

For the Euclidean \(k\)-means problem, it is well known that if \(k = 1\), the optimal center for a set of points \(X \subseteq \mathbb{R}^n\) is simply the centroid of \(X\). Prior work leverages one of the useful properties of Euclidean spaces as follows. Let \(\text{cost}(X, c)\) be the cost of clustering \(X\) using center \(c\). If \(C_X\) is the centroid of \(X\) (which is the optimal center), then the following identity is well known (see \cite{Inaba1994}):
\[ \text{cost}(X, c) = \text{cost}(X, C_X) + |X| \cdot \| c - C_X \|^2. \]

This identity provides a strong bound on the clustering cost when the center is estimated incorrectly because the distance between the optimal center and the estimated center is sufficient to bound the cost. However, this property does not hold for a general metric space. There, we can only apply the triangle inequality. Suppose we define \(\text{cost}(X, C_X) = \sum_{x \in X} d(x, C_X)^2\). For a center \(c\) at distance \(D := d(C_X, c)\) from \(C_X\), we can only conclude:
\[ \text{cost}(X, c) = \sum_{x \in X} d(x, c)^2 \le \sum_{x \in X} \bigl(d(x, C_X) + D\bigr)^2 
   = \text{cost}(X, C_X) + |X|\cdot D^2 + 2D \cdot \sum_{x \in X} d(x, C_X), \]
where the extra term \(D \cdot \sum_{x \in X} d(x, C_X)\) appears in contrast to the Euclidean case. Therefore, we found another way to relate the distance \(D\) between a center and the clustering cost \(\text{cost}(\cdot,\cdot)\). It takes the form of an inequality rather than an equality, but it is still sufficient for proving our main theorem.

Next, let us describe the idea behind our algorithm. Suppose the predictor labels a set of points \(P\) as one cluster, whereas the true optimal cluster is \(P^*\). If the predictor is sufficiently accurate, we assume:
\[
|P \cap P^*| \ge (1 - \alpha) \,\max\bigl(|P|,\,|P^*|\bigr).
\]
We prove (see Lemma~\ref{lemma:dominant-set-cost}) that if \(\alpha\) is sufficiently small (e.g., \(\alpha = 0.1\)), then
\(\text{cost}(P \cup P^*, C_{P^*})\) \text{ is close to}
\(\text{cost}(P \cup P^*, C_{P \cup P^*})\), which is optimal. From this, if we can approximate \(C_{P \cup P^*}\), the resulting center should be effective. However, we only know \(P\) and not \(P^*\). It turns out that by finding the center of the densest subset of \(P\), we effectively remove the impact of outliers in \(P \setminus P^*\) and return a suitable estimated center.

\section{Preliminary}

\subsection{Problem Definition}

Given a connected, weighted simple graph \( G = (V, E) \), the objective is to select a set of \( k \) points \( C \) (referred to as cluster centers) to minimize the cost function:  
\[
\text{cost}(G, C) = \sum_{v \in V} \min_{c \in C} \text{cost}(c, v),
\]
where \( \text{cost}(\cdot, \cdot) \) represents the cost associated with assigning a point \( v \) to a cluster centered at \( c \). Typically, we assume that \( \text{cost}(x, y) \) is a function of the distance \( d(x, y) \) between points \( x \) and \( y \) in \( V \).

\subsection*{Useful Notations}

\begin{itemize}
    \item \( d(x, y) \) denotes the distance between points \( x \) and \( y \) in \( V \).  
    \item \( \text{cost}(X, p) \) represents the total cost of a cluster made up of points in \( X \) and centered at \( p \). Formally:  
    \[
    \text{cost}(X, p) = \sum_{x \in X} \text{cost}(x, p).
    \]  
    \item \( C_X \) is the \emph{center} of the cluster \( X \), defined as:  
    \[
    C_X = \arg \min_{p \in V} \text{cost}(X, p).
    \]  
    \item \( B(x, r) \) denotes the \emph{ball} of radius \( r \), centered at \( x \in V \). Formally:  
    \[
    B(x, r) = \{ v \in V \mid d(x, v) \leq r \}.
    \]
\end{itemize}

\subsection*{Definition of k-Clustering in  General Metrics}

In a general metric graph, the \emph{k-clustering problem} involves partitioning the vertex set \( V \) into \( k \) disjoint subsets (clusters) such that each vertex \( v \in V \) is assigned to the cluster whose center minimizes the $l_q$ distance cost. Formally, the problem seeks:  
\[
OPT(k,q)=
\min_{C \subseteq V, |C| = k} \sum_{v \in V} \min_{c \in C} d(v, c)^q,
\]
where \( C \) is the set of \( k \) cluster centers and \( d(v, c) \) denotes the shortest path distance (or edge-weight-based distance) between vertices \( v \) and \( c \) in the graph \( G \).  

The clustering problem operates under the following assumptions:  
\begin{itemize}
    \item \( G \) is a connected graph, ensuring distances between any two vertices are well-defined.  
    \item \( k \) is a user-specified number of desired clusters.  
    \item The distance function \( d(\cdot, \cdot) \) satisfies the triangle inequality, as is typical in metric spaces.  
\end{itemize}

The goal is to minimize the total distance-based cost incurred by assigning each vertex \( v \in V \) to its closest cluster center in \( C \).

\subsection*{Learning-Augmented Setting}

In the \emph{learning-augmented setting}, we assume access to a predictor \( \Pi \) that provides labels for each point. These labels align with a \( (1+\alpha) \)-approximately optimal clustering \( C \). A predictor \( \Pi \) is said to have a \emph{label error rate} \( \lambda \leq \alpha \) if:  
\begin{itemize}
    \item For each cluster label \( i \in [k] := \{1, \cdots, k\} \), \( \Pi \) makes errors on at most a \( \lambda \) fraction of all points in the true cluster \( i \) of \( C \).  
    \item Additionally, \( \Pi \) errs on at most a \( \lambda \) fraction of all points it assigns to label \( i \).  
\end{itemize}
In other words, \( \Pi \) guarantees a precision and recall of at least \( (1-\lambda) \) for each label.  

Our goal is to find a set of centers \( \Tilde{C} \) such that:
\[
\text{cost}(G, \Tilde{C}) \leq (1 + f(\alpha)) \cdot \text{cost}(G, C^{\text{opt}}),
\]
where \( f: \mathbb{R}^{+} \rightarrow \mathbb{R}^{+} \) is a function of \( \alpha \). Ideally, as \( \alpha \rightarrow 0 \), we aim for \( f(\alpha) \rightarrow 0 \), though this may not always be achievable.

We note that our choice of predictor is both well-founded and practically relevant. Classical clustering problems such as k-median and k-means have been extensively studied in learning-augmented settings (\citep{learning_augmented_kmeans}, \citep{gamlath2022noisylabels}, \citep{theoretical_guarantees}), all of which assume the predictor model used in our paper. While exploring alternative predictors may lead to new insights, to the best of our knowledge, no prior work has employed fundamentally different predictors. Specifically, \citep{learning_augmented_kmeans} presents empirical evidence demonstrating that this predictor outperforms standard baselines such as `kmeans++' as well as alternative predictors like nearest neighbor and simple neural networks on real-world datasets.

\section{Learning-Augmented k-Clustering Algorithm}

In this section, we extend the learning augmented algorithm from Euclidean metrics to general metrics.
The following is the current candidate solutions of the above problem.

\noindent Comparison with Prior Work:
Our approach extends the learning-augmented $k$-means clustering algorithm proposed in~\citep{learning_augmented_kmeans} to general metric spaces. The prior work assumes that the input lies in Euclidean space and refines cluster centers coordinate-wise using a robust interval-based filtering method (\textsc{CrdEst}). While this method is efficient in Euclidean settings due to the separability of dimensions, it cannot be applied to general metric spaces. In contrast, our algorithm removes the reliance on Euclidean structure by employing a ball-based center estimation approach that minimizes clustering cost within a robustly chosen subset of points. While our work shares some conceptual similarities—such as using subset selection for cost minimization—the methods differ significantly. The approach in \cite{theoretical_guarantees} operates in a coordinate-wise manner: it selects a subset of points and computes the mean (centroid) along each coordinate independently, then bounds the distance between this centroid and the optimal center in each dimension. In contrast, our method selects a single subset of points in the full space and uses it directly, allowing the approach to generalize naturally to arbitrary metric spaces beyond Euclidean settings.
\begin{algorithm}[H]
\caption{Learning-augmented $k$ clustering in general metrics}
\label{alg:main}
\begin{algorithmic}[1]
\REQUIRE{A point set $X$ with labels from predictor $\Pi$, label error rate $\lambda$, and distance $l_q$}
\ENSURE{$(1+O(\alpha))$-approximate $k$-means clustering of $X$}
\FOR{$i = 1$ to $k$}
    \STATE{Let $X_i$ be the set of points in $X$ labeled as $i$ by $\Pi$.}
    \STATE{Run $\textsc{GetCenter}(X_i, \lambda)$ to compute the center $C'_i$.}
\ENDFOR
\STATE{\textbf{Return} $C'_1, \ldots, C'_k$ \quad \COMMENT{Output all cluster centers.}}
\end{algorithmic}
\end{algorithm}

\begin{algorithm}[H]
\caption{\textsc{GetCenter}}
\label{alg:getcenter}
\begin{algorithmic}[1]
\REQUIRE{A point set $X$, corruption level $\lambda \leq \alpha$, and distance $l_q$}
\IF{$q = 1$}
    \STATE{Return $C_X$ \quad \COMMENT{Just return the center of a predictor output}}
\ELSIF{$q = 2$}
    \STATE{Compute $S \gets \argmin_{S \subset V(G), \card{S} = (1-\alpha) \card{V(G)}} cost(S, C_S)$ \quad \COMMENT{Estimate center.}}
    \STATE{\textbf{Return} $C_S$ \quad \COMMENT{Return the center estimate.}}
\ENDIF
\end{algorithmic}
\end{algorithm}

\begin{theorem}\label{Thm:appro}
For $k$-clustering problem with respect to the $\ell_q \, (q=1,2)$ distance in a general metric space, let $\alpha \in \left(0, \frac{1}{2} \right)$, and $\Pi$ be a predictor with label error rate $\lambda \leq \alpha$. Then \textbf{Algorithm 1} outputs a $(1 + O(\alpha^{1/q}))$-approximation in polynomial time.
\end{theorem}

\subsection{Proof of \textbf{Theorem \ref{Thm:appro}}}

We first prove \textbf{Theorem \ref{Thm:appro}}, which shows that \textbf{Algorithm 1} provides a \((1 + O(\alpha))\)-approximation to the optimal \(k\)-clustering in general metrics. Specifically, we demonstrate that the empirical center computed from any \((1 - \alpha)\)-fraction of the points serves as a robust approximation to the true center.

\begin{lemma}
    Let \( P \) and \( Q \) be disjoint subsets of \( V \) such that \( X = P \cup Q \), \( |P| \geq (1 - \alpha) |X| \) and \( |Q| \leq \alpha |X| \). Then,
    \[
    \text{cost}(X, C_P) \leq (1 + f(\alpha)) \cdot \text{cost}(X, C_X),
    \]
    for some function \( f \). Especially, for $\alpha < 1/8$, the following holds:
    \begin{itemize}
        \item If \( \text{cost}(X, p) = \sum_{x \in X} d(x, p) \), then \( f(\alpha) = \frac{2\alpha}{1 - \alpha} \).
        \item If \( \text{cost}(X, p) = \sum_{x \in X} d(x, p)^2 \), then \( f(\alpha) = 6\sqrt{\frac{\alpha}{1 - \alpha}} \).
    \end{itemize}
    \label{lemma:dominant-set-cost}
\end{lemma}

\begin{proof}
    For any point \( p \in P \), by the triangle inequality, we have  
    \[
    D := d(C_P, C_X) \leq d(C_P, p) + d(p, C_X).
    \]

    Case 1: \( \text{cost}(X, p) = \sum_{x \in X} d(x, p) \).
    Summing this inequality over all points in \( P \), we obtain:  
    \[
    |P| \cdot D \leq \sum_{p \in P} \left( d(C_P, p) + d(p, C_X) \right) = \text{cost}(P, C_P) + \text{cost}(P, C_X) \leq 2 \cdot \text{cost}(P, C_X),
    \]  
    since \( \text{cost}(P, C_P) \leq \text{cost}(P, C_X) \) by definition of \( C_P \).  
    Similarly, for any \( q \in Q \), by the triangle inequality,  
    \[
    d(q, C_P) \leq d(q, C_X) + d(C_X, C_P).
    \]  
    Therefore,  
    \[
    \text{cost}(Q, C_P) \leq \text{cost}(Q, C_X) + |Q| \cdot D.
    \]
    Combining these results, we obtain an upper bound for \( \text{cost}(X, C_P) \):
    \begin{align}
        \text{cost}(X, C_P) &= \text{cost}(P, C_P) + \text{cost}(Q, C_P) \notag \\
        &\leq \text{cost}(P, C_X) + \text{cost}(Q, C_X) + |Q| \cdot D \notag \\
        &= \text{cost}(X, C_X) + |Q| \cdot D. \notag
    \end{align}
    Substituting \( D \leq \frac{2 \cdot \text{cost}(P, C_X)}{|P|} \), we get:  
    \[
    \text{cost}(X, C_P) \leq \text{cost}(X, C_X) + \frac{2 |Q|}{|P|} \cdot \text{cost}(P, C_X).
    \]
    Using \( |Q| \leq \alpha |X| \) and \( |P| \geq (1 - \alpha)|X| \), we have  
    \[
    \frac{|Q|}{|P|} \leq \frac{\alpha}{1 - \alpha}.
    \]
    Substituting this back, we obtain:  
    \[
    \text{cost}(X, C_P) \leq \left( 1 + \frac{2\alpha}{1 - \alpha} \right) \text{cost}(X, C_X).
    \]

    Case 2: \( \text{cost}(X, p) = \sum_{x \in X} d(x, p)^2 \). 
    For squared distances, by a similar argument, we start with:  
    \[
    \text{cost}(X, C_P) - \text{cost}(X, C_X) \leq \sum_{q \in Q} \left( (d(q, C_X) + D)^2 - d(q, C_X)^2 \right).
    \]
    Expanding the square terms, we get:  
    \[
    \text{cost}(X, C_P) - \text{cost}(X, C_X) \leq 2D \cdot \sum_{q \in Q} d(q, C_X) + |Q| \cdot D^2.
    \]
    In this case, by the Cauchy-Schwarz inequality, the following holds:
    \begin{align}
        |P| \cdot D \leq \sum_{p \in P} \left( d(C_P, p) + d(p, C_X) \right) &\leq \sqrt{|P| \cdot \text{cost}(P, C_P)} + \sqrt{|P| \cdot \text{cost}(P, C_X)} \notag \\
        &\leq 2 \sqrt{|P| \cdot \text{cost}(X, C_X)}, \notag
    \end{align}
    and
    \[
    \sum_{q \in Q} d(q, C_X) \leq \sqrt{|Q| \cdot \text{cost}(Q, C_X)}.
    \]
    Therefore, by substituting $D$ we obtain:
    \[
    \text{cost}(X, C_P) - \text{cost}(X, C_X) \leq 4 \sqrt{\frac{|Q|}{|P|}} \cdot \text{cost}(X, C_X) + \frac{4\card{Q}}{\card{P}} \cdot \text{cost}(X, C_X).
    \]
    For \( \alpha < \frac{1}{8} \), combining terms gives:
    \[
    \text{cost}(X, C_P) \leq \left( 1 + 6\sqrt{\frac{\alpha}{1 - \alpha}} \right) \text{cost}(X, C_X).
    \]
    This completes the proof.
\end{proof}

When \( \text{cost} \) is defined as the sum of distances, we proved that the predictor's output is good enough to use without any modification.

\begin{lemma}
    Let \( P_i \) be the set of points which the predictor labels as \( i \), and \( P^*_i \) be the set of points belonging to the \( i \)-th set in the optimal solution. Then for $\alpha < 1/2$,
    \[
    \text{cost}(P^*_i, C_{P_i}) \leq \left(1 + \frac{4\alpha}{1 - 2\alpha}\right) \cdot \text{cost}(P^*_i, C_{P^*_i})
    \]
    when \( \text{cost}(X, p) = \sum_{x \in X} d(x, p) \).
\end{lemma}

\begin{proof}
    The label names are not important, so we abuse the notation by ignoring subscripts. Let $D := d(C_P, C_{P^*})$. For any point \( p \in P \cap P^* \), by the triangle inequality, we have  
    \[ D = d(C_P, C_{P^*}) \leq d(C_P, p) + d(p, C_{P^*}). \]

    Summing this inequality over all points in $P \cap P^*$, we obtain:
    \begin{align}
        \card{P \cap P^*} \cdot D \le \text{cost}(P \cap P^*, C_P) + \text{cost}(P \cap P^*, C_{P^*}).
        \label{eq:D1}
    \end{align}

    From the definition of the center $C_P$, $\text{cost}(P, C_P) \le \text{cost}(P, C_{P^*})$ holds. Since $P = (P \cap P^*) \cup (P \setminus P^*)$,
    \[ \text{cost}(P \cap P^*, C_P) \le \text{cost}(P \cap P^*, C_{P^*}) + \text{cost}(P \setminus P^*, C_{P^*}) - \text{cost}(P \setminus P^*, C_P). \]
    Moreover, by the triangle inequality $d(p, C_{P^*}) \le d(p, C_P) + D$ for any $p \in P \setminus P^*$, the following holds:
    \begin{align}
        \text{cost}(P \cap P^*, C_P) \le \text{cost}(P \cap P^*, C_{P^*}) + \card{P \setminus P^*} \cdot D.
        \label{eq:D2}
    \end{align}

    Combining \eqref{eq:D1} and \eqref{eq:D2} gives the upper bound for $D$:
    \[ \card{P \cap P^*} \cdot D \le 2 \cdot \text{cost}(P \cap P^*, C_{P^*}) + \card{P \setminus P^*} \cdot D \]
    which implies
    \[ D \le \frac{2 \cdot \text{cost}(P \cap P^*, C_{P^*})}{\card{P \cap P^*} - \card{P \setminus P^*}} \le \frac{2 \cdot \text{cost}(P \cap P^*, C_{P^*})}{(1-2\alpha) \max (\card{P}, \card{P^*})} \]
    because $\card{P \cap P^*} \ge (1-\alpha) \cdot \max (\card{P}, \card{P^*})$.

    At this point, we can derive the result as follows. By \eqref{eq:D2} and the triangle inequality, we have
    \begin{align}
        \text{cost}(P^*, C_P) &= \text{cost}(P \cap P^*, C_P) + \text{cost}(P^* \setminus P, C_P) \notag \\
        &\le \text{cost}(P \cap P^*, C_{P^*}) + \card{P \setminus P^*} \cdot D + \text{cost}(P^* \setminus P, C_P) \notag \\
        &\le \text{cost}(P \cap P^*, C_{P^*}) + \card{P \setminus P^*} \cdot D + \text{cost}(P^* \setminus P, C_{P^*}) + \card{P^* \setminus P} \cdot D \notag \\
        &\le \text{cost}(P^*, C_{P^*}) + 2\alpha \cdot \max (\card{P}, \card{P^*}) \cdot D \notag
    \end{align}
    because $\card{P \cap P^*} \ge (1-\alpha) \cdot \max (\card{P}, \card{P^*})$. Plugging into the above upper bound for $D$ gives the desired result.
\end{proof}

When \(\text{cost}\) is defined as the sum of squared distances, the proof takes a slightly different approach.
First, we identify a subset \( B \subset P \), which is the ball of size \((1-\alpha)|P|\) that minimizes the clustering cost, i.e.,  
\[
\text{cost}(B, C_B) = \min_{\substack{S \subset P \\ |S| = (1-\alpha)|P|}} \text{cost}(S, C_S),
\]
where \( C_B \) is the center of \( B \).
Since \( \card{P \cap P^*} \subset P\) and \(|P \cap P^*| \geq (1-\alpha)|P|\), it follows that:  
\[
\text{cost}(B, C_B) \leq \text{cost}(P \cap P^*, C_{P \cap P^*}).
\]
Our claim is that the center \( C_B \) serves as the desired empirical center. Intuitively, \( B \) captures the majority of the points in \( P \), ensuring a clustering cost close to optimal. Actually, we have slightly misused the term \textit{ball} here. This is because if $\card{B}$ is fixed as $(1-\alpha) \card{P}$, then there may not exist a ball $B$ that satisfies this condition. Therefore, we refer to $B(x, r)$ as a \emph{semi-ball} of radius $r$ centered at $x$, which means that $B$ contains all points $p$ such that $d(p, x) < r$, as well as \emph{some} points $p$ for which $d(p, x) = r$.

 \begin{lemma}
    Let \( P_i \) be the set of points which the predictor labels as \( i \), and \( P^*_i \) be the set of points belonging to the \( i \)-th set in the optimal solution. Let $C_{B_i}$ be the center of semi-ball $B_i$ of size $(1-\alpha) \card{P_i}$ that minimizes the clustering cost.
    Moreover, when the cost is defined as the sum of the squares of distances, 
    \[
    \text{cost}(P^*_i, C_{B_i}) \leq \left(1 + O\left(\sqrt{\alpha}\right)\right) \cdot \text{cost}(P^*_i, C_{P^*_i}).
    \]
\end{lemma}

\begin{proof}
    The label names are not important, so we abuse the notation by ignoring subscripts. Let \( R = B \cap P \cap P^* \), \( R_1 = B \setminus R \), and \( R_2 = (P \cap P^*) \setminus R \). Also, define \( X = B \cup (P \cap P^*) \). By Lemma~\ref{lemma:dominant-set-cost}, we have:
    \begin{align}
        \text{cost}(X, C_B) \leq \left(1 + 6\sqrt{\frac{\alpha}{1-\alpha}}\right) \cdot \text{cost}(X, C_X) 
        \leq \left(1 + 6\sqrt{\frac{\alpha}{1-\alpha}}\right) \cdot \text{cost}(X, C_{P \cap P^*}),
        \label{eq:2}
    \end{align}
    because \( |B| \geq (1-\alpha)|P| \) and \( |X| \leq |P| \). 
    For brevity, let \( k_1 = \frac{6}{\sqrt{1-\alpha}} \). Let \( D = d(C_{P \cap P^*}, C_B) \). By the triangle inequality, for \( r \in R_1 \), we have \( d(r, C_{P \cap P^*}) \leq d(r, C_B) + D \). Thus:
    \begin{align}
        \text{cost}(R_1, C_{P \cap P^*}) = \sum_{r \in R_1} d(r, C_{P \cap P^*})^2 
        &\leq \sum_{r \in R_1} (d(r, C_B) + D)^2 \notag \\
        &= \text{cost}(R_1, C_B) + 2D \sum_{r \in R_1} d(r, C_B) + |R_1| \cdot D^2.
        \label{eq:3}
    \end{align}
    On the other hand, by the triangle inequality, for \( r \in R \), it holds that:
    \[
    D = d(C_{P \cap P^*}, C_B) \leq d(C_{P \cap P^*}, r) + d(r, C_B).
    \]
    Summing over all \( r \in R \) and applying the Cauchy-Schwarz inequality gives:
    \begin{align}
        |R| \cdot D = \sum_{r \in R} d(C_{P \cap P^*}, C_B) 
        &\leq \sum_{r \in R} \left( d(C_{P \cap P^*}, r) + d(C_B, r) \right) \notag \\
        &\leq \sqrt{|R|} \cdot \left( \sqrt{\text{cost}(R, C_{P \cap P^*})} + \sqrt{\text{cost}(R, C_B)} \right).
        \label{eq:4}
    \end{align}
    
   Our first goal is to find an upper bound for \( \text{cost}(P \cap P^*, C_B) = \text{cost}(R \cup R_2, C_B) \). By \eqref{eq:2} and \eqref{eq:3}, we have
    \begin{align}
        \text{cost}(P \cap P^*, C_B) &= \text{cost}(X, C_B) - \text{cost}(R_1, C_B) \notag \\
        &\leq (1 + k_1 \sqrt{\alpha}) \cdot \text{cost}(X, C_{P \cap P^*}) - \text{cost}(R_1, C_B) \notag \\
        &\leq (1 + k_1 \sqrt{\alpha}) \cdot \text{cost}(R \cup R_2, C_{P \cap P^*}) 
        + k_1 \sqrt{\alpha} \cdot \text{cost}(R_1, C_B) \notag \\
        &\quad + (1 + k_1 \sqrt{\alpha}) \cdot \left( 2D \cdot \sum_{r \in R_1} d(r, C_B) 
        + |R_1| \cdot D^2 \right). \notag
    \end{align}
    By the Cauchy-Schwarz inequality, \( \sum_{r \in R_1} d(r, C_B) \leq \sqrt{|R_1| \cdot \text{cost}(R_1, C_B)} \) holds. Also, by substituting $D$ by \eqref{eq:4}, the above inequality becomes as follow:
    \begin{align}
        &\qquad \text{cost}(P \cap P^*, C_B) \notag \\
        &\leq (1 + k_1 \sqrt{\alpha}) \cdot \text{cost}(R \cup R_2, C_{P \cap P^*}) 
        + k_1 \sqrt{\alpha} \cdot \text{cost}(R_1, C_B) \notag \\
        &\quad + 2 (1 + k_1 \sqrt{\alpha}) \cdot \sqrt{\frac{|R_1|}{|R|}}
        \left( \sqrt{\text{cost}(R, C_B) \cdot \text{cost}(R_1, C_B)} 
        + \sqrt{\text{cost}(R, C_{P \cap P^*}) \cdot \text{cost}(R_1, C_B)} \right) \notag \\
        &\quad + (1 + k_1 \sqrt{\alpha}) \cdot \frac{|R_1|}{|R|} 
        \left( \sqrt{\text{cost}(R, C_{P \cap P^*})} + \sqrt{\text{cost}(R, C_B)} \right)^2 \notag \\
        &\leq (1 + k_1 \sqrt{\alpha}) \cdot \text{cost}(P \cap P^*, C_{P \cap P^*}) 
        + k_1 \sqrt{\alpha} \cdot \text{cost}(R_1, C_B) \notag \\
        &\quad + (1 + k_1 \sqrt{\alpha}) \cdot \sqrt{\frac{\alpha}{1 - \alpha}}
        \left( \text{cost}(R, C_B) + \text{cost}(R_1, C_B) 
        + \text{cost}(R, C_{P \cap P^*}) + \text{cost}(R_1, C_B) \right) \notag \\
        &\quad + 2 (1 + k_1 \sqrt{\alpha}) \cdot \frac{\alpha}{1 - \alpha} 
        \left( \text{cost}(R, C_{P \cap P^*}) + \text{cost}(R, C_B) \right)
        \label{eq:5}
    \end{align}
    where the last inequality follows from the simple AM-GM inequality $2\sqrt{ab} \leq a+b$.
    
    Our second goal is to bound the remaining term \( \text{cost}(P^* \setminus P, C_B) \). By the triangle inequality,
    \begin{align}
        \text{cost}(P^* \setminus P, C_B) &= \sum_{p \in P^* \setminus P} d(p, C_B)^2 \notag \\
        &\leq \sum_{p \in P^* \setminus P} \left( d(p, C_{P \cap P^*}) + D \right)^2 \notag \\
        &= \text{cost}(P^* \setminus P, C_{P \cap P^*}) 
        + 2D \cdot \sum_{p \in P^* \setminus P} d(p, C_{P \cap P^*}) 
        + |P^* \setminus P| \cdot D^2. \notag
    \end{align}
    Note that \( |R| = |B \cap P \cap P^*| = |B| + |P \cap P^*| - |B \cup (P \cap P^*)| \geq |P \cap P^*| - \alpha |P| \). Also, since \( (1-\alpha) |P| \leq |P \cap P^*| \leq |P^*| \), thus \( |R| \geq \frac{1-2\alpha}{1-\alpha} |P^*| \).  
    Substituting $D$ by \eqref{eq:4} and applying the Cauchy-Schwarz inequality, we obtain:
    \begin{align}
        &\quad \text{cost}(P^* \setminus P, C_B) \notag \\
        &\leq \text{cost}(P^* \setminus P, C_{P \cap P^*}) \notag \\
        &\quad + 2 \sqrt{\frac{|P^* \setminus P|}{|R|}}
        \left( \sqrt{\text{cost}(R, C_B) \cdot \text{cost}(P^* \setminus P, C_{P \cap P^*})} 
        + \sqrt{\text{cost}(R, C_{P \cap P^*}) \cdot \text{cost}(P^* \setminus P, C_{P \cap P^*})} \right) \notag \\
        &\quad + \frac{|P^* \setminus P|}{|R|} 
        \left( \sqrt{\text{cost}(R, C_B)} + \sqrt{\text{cost}(R, C_{P \cap P^*})} \right)^2 \notag \\
        &\leq \text{cost}(P^* \setminus P, C_{P \cap P^*}) \notag \\
        &\quad + \sqrt{\frac{\alpha (1-\alpha)}{1-2\alpha}} 
        \left( \text{cost}(R, C_B) + \text{cost}(P^* \setminus P, C_{P \cap P^*}) 
        + \text{cost}(R, C_{P \cap P^*}) + \text{cost}(P^* \setminus P, C_{P \cap P^*}) \right) \notag \\
        &\quad + \frac{2\alpha (1-\alpha)}{1-2\alpha} 
        \left( \text{cost}(R, C_B) + \text{cost}(R, C_{P \cap P^*}) \right).
        \label{eq:6}
    \end{align}
    where the last inequality follows from the simple AM-GM inequality $2\sqrt{ab} \leq a+b$.
    Finally, for \( \alpha < \frac{1}{8} \), adding \eqref{eq:5} and \eqref{eq:6} gets us:
    \begin{align}
        \text{cost}(P^*, C_B) 
        &\leq \left( 1 + k_1 \sqrt{\alpha} \right) \cdot \text{cost}(P \cap P^*, C_{P \cap P^*}) \notag \\
        &\quad + \left( 2(1 + k_1 \sqrt{\alpha}) \sqrt{\frac{\alpha}{1-\alpha}} + k_1 \sqrt{\alpha} \right) 
        \cdot \left( \text{cost}(R, C_B) + \text{cost}(R_1, C_B) \right) \notag \\
        &\quad + \left( (1 + k_1 \sqrt{\alpha}) \left( \sqrt{\frac{\alpha}{1-\alpha}} + \frac{2\alpha}{1-\alpha} \right) 
    + \sqrt{\frac{\alpha(1-\alpha)}{1-2\alpha}} + \frac{2\alpha(1-\alpha)}{1-2\alpha} \right) \notag \\
    &\qquad \cdot \left( \text{cost}(R, C_{P \cap P^*}) + \text{cost}(P^* \setminus P, C_{P \cap P^*}) \right).
\end{align}

Since \( \text{cost}(R, C_B) + \text{cost}(R_1, C_B) \leq \text{cost}(B, C_B) \leq \text{cost}(P \cap P^*, C_{P \cap P^*}) \leq \text{cost}(P^*, C_{P^*}) \),  
and by Lemma \ref{lemma:dominant-set-cost},  
\[
\text{cost}(R, C_{P \cap P^*}) + \text{cost}(P^* \setminus P, C_{P \cap P^*}) 
\leq \text{cost}(P^*, C_{P \cap P^*}) \leq (1 + k_1 \sqrt{\alpha}) \cdot \text{cost}(P^*, C_{P^*}).
\]
Therefore, for $\alpha < 1/8$, it can be proved that:
\[
\text{cost}(P^*, C_B) \leq (1 + 45\sqrt{\alpha}) \cdot \text{cost}(P^*, C_{P^*}).
\]
However, we have not attempted to optimize this constant factor, and it seems there is room for improvement.
\end{proof}

\section{Query Complexity Lower Bound}

In this paper, we define query complexity as follows: the predictor assigns expected labels to input points, and obtaining the label of a point is considered issuing one query. This definition is consistent with prior work in \citep{learning_augmented_kmeans}. In this section, we investigate the hardness of query complexity in clustering. The goal is to determine whether we can achieve the same performance as the main theorem (full-information scenario), while only predicting labels for a subset of points. In the Euclidean setting, \citep{learning_augmented_kmeans} shows that for any $\delta \in (0, 1]$, any $(1+\alpha)$-approximation algorithm runs in time $O(2^{n^{1-\delta}})$ must make $\omega(\frac{k^{1-\delta}}{\alpha \log k})$ queries under the ETH. Their proof proceeds in two steps, building on the result of \cite{lee2017improved}, which improves the inapproximability of the \(k\)-means problem by a reduction from the vertex cover problem on $4$-regular graph.

First, they prove that the vertex cover problem remains APX-hard even when a sublinear portion (e.g., of size \(O(n^{1-\delta})\)) of a minimum vertex cover in an \(n\)-vertex graph \(G\) is known. Second, they construct a \(k\)-means instance that embeds the instance constructed by \cite{lee2017improved}, thereby translating the ability to query the predictor into revealing partial solutions of the optimal vertex cover. The hardness of query complexity can be intuitively understood as follows: even if a part of the optimal solution is known (e.g., corresponding to $O(k^{1-\delta})$ centers), it remains impossible to approximate the vertex cover arbitrarily closely to $1$ (as $\alpha \rightarrow 0$) because the remaining problem instance, stripped of the known part, still corresponds to a vertex cover problem, which is APX-hard.

We show similar results to those in \citep{learning_augmented_kmeans} also hold in a general metric setting. Our proof starts from establishing a reduction from the vertex cover problem on $4$-regular graphs to our $k$-clustering problem. \citep{lee2017improved} converts the 4-regular graph into a Euclidean $k$-means instance, which can be used to approximate the minimum vertex cover of the original graph. Its NP-hardness is derived from the results of \citep{chlebik2006complexity}, which showed that it is NP-hard to approximate the minimum vertex cover problem within a factor of $\alpha_{max} / \alpha_{min}$, where $\alpha_{min} = (2 \mu_{4,k}+8)/(4\mu_{4,k} + 12)$ and $\alpha_{max} = (2 \mu_{4,k}+9)/(4\mu_{4,k} + 12)$ for $\mu_{4,k} \leq 21.7$.

Moreover, in \citep{learning_augmented_kmeans}, the authors mentioned that assuming ETH, approximate it within the same approximation ratio in time $O(2^{n^{1-\delta}})$ is also impossible. We note that we use this more strong result in the proof of Theorem~\ref{thm:lower-bound-main}.

The conversion in \citep{lee2017improved}, from a $4$-regular graph $G$ to $G'$, is as follows: since $\card{E(G)} = 2n$, and given the well-known fact that every graph has a cut of size at least half of its edges, there exists a set $E_2 \subseteq E(G)$ such that $\card{E_2} = n$. Let $E_1 = E(G) \setminus E_2$. Then $G'$ is created by replicating each edge in $E_1$ three times, specifically:

$$ V(G') = V(G) \bigcup_{e=(u, v) \in E_1} \{ v'_{e,u}, v'_{e,v} \}, $$
$$ E(G') = E(G) \bigcup_{e=(u, v) \in E_1} \{ (u, v'_{e,u}), (v'_{e,u}, v'_{e,v}), (v'_{e,v}, v) \}, $$
which implies that $\card{V(G')} = 3n$, $\card{E(G')} = 4n$. 
\begin{lemma}
    [\citep{lee2017improved}] For a given $4$-regular graph $G$ and its converted graph $G'$, the following holds:
    \begin{itemize}
        \item If the size of the minimum vertex cover of $G$ is at most $\alpha_{min} n$, then $G'$ has a vertex cover of size at most $(\alpha_{min}+1) n$.
        \item If every vertex cover of $G$ has a size of at least $\alpha_{max} n$, then every vertex cover of $G'$ has a size of at least $(\alpha_{max}+1) n$.
    \end{itemize}
    \label{lemma:gprime-vertex-cover}
\end{lemma}

Our goal is to construct a $k$-clustering problem instance, a weighted graph $G''=(V(G''), E(G''))$, as follows. Let $V(G'') = \{ v_{e'} \, | \, e' \in E(G') \} \cup \{ s \}$, meaning that the nodes of $G''$ correspond to the edges of $G'$, plus a special node $s$. For $E(G'')$, there are two types of edges. First, an edge $(v_{e_1}, v_{e_2})$ with weight $1$ exists if and only if $e_1$ and $e_2$ are incident in $G'$. Second, for any node $u \in V(G'')$, an edge $(u, s)$ with weight $2^{1/q}$ exists. In summary, $G''$ is the union of the line graph of $G'$, $L(G')$ (with weight $1$), and a star centered at $s$ with weight $2^{1/q}$. Also, let $k = (\alpha_{min}+1) n$. The lemma from \citep{lee2017improved}, originally proposed in \citep{awasthi2015hardness}, states that a similar property holds in $G''$.
\begin{lemma}
    Let $X = \{ v_{e_1}, \ldots, v_{e_l} \}$ be a cluster on $G''$. Then $l-1 \leq cost(X, C_X) \leq 2l$, and there exist two nodes in $V(G')$ that are incident with at least $2l-1-cost(X, C_X)$ edges of $G'$. Furthermore, $cost(X, C_X) = l-1$ if and only if $X$ is a star.
    \label{lemma:center-concentrate}
\end{lemma}
\begin{proof}
    When the special node $s$ is selected as the center, we have $cost(X, s) = 2l \geq cost(X, C_X)$. Since any distance between two distinct nodes in $V(G'')$ is at least 1, $cost(X, C_X) \geq l-1$. Additionally, if $cost(X, C_X) = l-1$, then at least one node contributes a cost of $0$, implying $C_X \in X$. Other nodes in $X \setminus { C_X }$ are adjacent to $C_X$ in $G''$, meaning $X$ is a star in $G'$.

    When $cost(X, C_X) < 2l-1$, if $C_X \notin X$, there are at least $2l-cost(X, C_X)$ nodes at a distance of exactly 1 from $C_X$. Otherwise, if $C_X \in X$, then at least $2l-2-cost(X, C_X)$ nodes are at distance 1 from $C_X$, with $C_X$ itself at distance 0. Since the edges of $G''$ with weight 1 form the line graph of $G'$, there exists an edge $e$ such that $C_X = v_e$. Consequently, at least $2l-1-cost(C_X)$ edges are incident to $v_e$, so the endpoints of $e$ are the desired two nodes in $V(G')$.
\end{proof}

\begin{lemma}
    [Completeness] For $q=1, 2$, if $G$ has a vertex cover of size at most $\alpha_{min} n$, then the optimal clustering cost of $G''$ is at most $(3-\alpha_{min}) n + 2^{1/q}$.
    \label{lemma:completeness-k-clustering}
\end{lemma}
\begin{proof}
    By Lemma~\ref{lemma:gprime-vertex-cover}, there exists a vertex cover $S'$ of $G'$ with size at most $(\alpha_{min}+1) n$. We construct a clustering as follows: each cluster corresponds to the nodes in $S'$. Each $e \in E(G')$ is grouped with one of its endpoints in $S'$ (chosen arbitrarily if both endpoints are in $S'$). By construction, each cluster is star-shaped in $G'$. The special node $s \in G''$ is connected to all other nodes, contributing at most $2^{1/q}$ to the clustering cost. Therefore, the total clustering cost is:
    $$ 2^{1/q} + \sum_{i=1}^{k} cost(C_i) = 2^{1/q} + \sum_{i=1}^{k} (\card{C_i} - 1) = 4n-k = (3-\alpha_{min})n + 2^{1/q}. $$
\end{proof}

\begin{lemma}
    [Soundness] If every vertex cover of $G$ has a size of at least $\alpha_{max} n$, then the optimal clustering cost of $G''$ is at least $(3-\alpha_{min} + \frac{1}{3} (\alpha_{max}-\alpha_{min})) n$.
    \label{lemma:soundness-k-clustering}
\end{lemma}
\begin{proof}
    Ignoring the special node $s$, let $\mathcal{C}$ be any clustering of $G''-s$. By Lemma~\ref{lemma:center-concentrate} and Lemma~\ref{lemma:gprime-vertex-cover}, we can apply the proof of Lemma 5 from \citep{lee2017improved}. Since the contribution of $s$ to the clustering cost is non-negative, the result holds.
\end{proof}

By Lemma~\ref{lemma:completeness-k-clustering} and Lemma~\ref{lemma:soundness-k-clustering}, it is NP-hard to distinguish the following cases:
\begin{itemize}
    \item The optimal cost of the $k$-clustering problem is at most $(3-\alpha_{min}) n + 2^{1/q}$.
    \item Every feasible clustering cost is at least $(3-\alpha_{min} + \frac{1}{3} (\alpha_{max}-\alpha_{min})) n$.
\end{itemize}

Therefore, there exists a constant $C$ such that it is NP-hard to approximate the general metric $k$-clustering problem on a graph within a factor of $(1+C)$. We establish the following hardness result:
\begin{theorem}
    For $\delta \in (0, 1)$ and $q=1, 2$, no algorithm can produce $(1+O(\alpha^{1/q}))$-approximation solutions to the optimal $k$-clustering problem on a general metric in $O(2^{n^{1-\delta}})$ time while performing at most $O(\frac{k^{1-\delta}}{\log k})$ queries to the predictor $\Pi$ with label error rate $\alpha$, assuming the Exponential Time Hypothesis.
    \label{thm:lower-bound-main}
\end{theorem}
\begin{proof}
    The proof is by contradiction. Let $\mathcal{I}$ be the instance of the general metric $k$-clustering problem constructed via the process described above. Assume for contradiction that there exists an algorithm running in time $O(2^{n^{1-\delta}})$ and making at most $O(\frac{k^{1-\delta}}{\log k})$ queries, which achieves a $(1+O(\alpha^{1/q}))$-approximation. Then, we can convert this algorithm into one that does not use the predictor. This is because we can simulate all $k^{O(\frac{k^{1-\delta}}{\log k})} = O(2^{k^{1-\delta}})$ possible predictor outputs. For each set of simulated outputs, we can compute the corresponding solution for the base problem (vertex cover on $4$-regular graphs) and select the one with minimum cost. Thus, when $\alpha$ is sufficiently small, the properties of our reduction used to construct $\mathcal{I}$ imply the existence of an approximation algorithm for the vertex cover problem on $4$-regular graphs that surpasses its known approximation hardness, runs in time $O(2^{n^{1-\delta}})$.
\end{proof}

\acks{This research was partially supported by a startup fund from Seoul National University (SNU).}

\bibliographystyle{plainnat}
\bibliography{references}

\begin{thebibliography}{29}
\providecommand{\natexlab}[1]{#1}
\providecommand{\url}[1]{\texttt{#1}}
\expandafter\ifx\csname urlstyle\endcsname\relax
  \providecommand{\doi}[1]{doi: #1}\else
  \providecommand{\doi}{doi: \begingroup \urlstyle{rm}\Url}\fi

\bibitem[Ahmadian et~al.(2017)Ahmadian, Norouzi-Fard, Svensson, and Ward]{ahmadian2017primaldual}
Sara Ahmadian, Ashkan Norouzi-Fard, Ola Svensson, and Justin Ward.
\newblock Better guarantees for $k$-means and euclidean $k$-median by primal-dual algorithms.
\newblock In \emph{58th Annual IEEE Symposium on Foundations of Computer Science (FOCS)}, pages 61--72. IEEE, 2017.
\newblock \doi{10.1109/FOCS.2017.16}.
\newblock URL \url{https://arxiv.org/abs/1612.07925}.

\bibitem[Arthur and Vassilvitskii(2007)]{arthur2007kmeanspp}
David Arthur and Sergei Vassilvitskii.
\newblock k-means++: The advantages of careful seeding.
\newblock In \emph{Proceedings of the 18th Annual ACM-SIAM Symposium on Discrete Algorithms (SODA)}, pages 1027--1035, 2007.

\bibitem[Awasthi et~al.(2015)Awasthi, Charikar, Krishnaswamy, and Sinop]{awasthi2015hardness}
Pranjal Awasthi, Moses Charikar, Ravishankar Krishnaswamy, and Ali~Kemal Sinop.
\newblock The hardness of approximation of euclidean k-means.
\newblock In \emph{Proceedings of the 31st Annual Symposium on Computational Geometry (SoCG)}, 2015.

\bibitem[Bachem et~al.(2018)Bachem, Lucic, and Krause]{bachem2018scalable}
Olivier Bachem, Mario Lucic, and Andreas Krause.
\newblock Scalable and provably accurate algorithms for clustering with outliers.
\newblock In \emph{Proceedings of the 35th International Conference on Machine Learning (ICML)}, pages 322--331, 2018.

\bibitem[Chleb\'{\i}k and Chleb\'{\i}kov\'{a}(2006)]{chlebik2006complexity}
Miroslav Chleb\'{\i}k and Janka Chleb\'{\i}kov\'{a}.
\newblock Complexity of approximating bounded variants of optimization problems.
\newblock \emph{Theoretical Computer Science}, 354\penalty0 (3):\penalty0 320--338, 2006.

\bibitem[Cohen-Addad et~al.(2020)Cohen-Addad, Karthik, and Lee]{cohenaddad2020inapproximability}
Vincent Cohen-Addad, C.~S. Karthik, and Euiwoong Lee.
\newblock On approximability of clustering problems without candidate centers.
\newblock \emph{arXiv preprint arXiv:2010.00087}, 2020.
\newblock URL \url{https://arxiv.org/abs/2010.00087}.

\bibitem[Cohen-Addad et~al.(2021)Cohen-Addad, Karthik, and Lee]{cohenaddad2021johnsoncoverage}
Vincent Cohen-Addad, C.~S. Karthik, and Euiwoong Lee.
\newblock Johnson coverage hypothesis: Inapproximability of k-means and k-median in $\ell_p$ metrics.
\newblock \emph{arXiv preprint arXiv:2111.10912}, 2021.
\newblock URL \url{https://arxiv.org/abs/2111.10912}.

\bibitem[Cohen-Addad et~al.(2022{\natexlab{a}})Cohen-Addad, Esfandiari, Mirrokni, and Narayanan]{cohenaddad2022improved}
Vincent Cohen-Addad, Hossein Esfandiari, Vahab Mirrokni, and Shyam Narayanan.
\newblock Improved approximations for euclidean $k$-means and $k$-median, via nested quasi-independent sets.
\newblock \emph{arXiv preprint arXiv:2204.04828}, 2022{\natexlab{a}}.
\newblock URL \url{https://arxiv.org/abs/2204.04828}.

\bibitem[Cohen-Addad et~al.(2022{\natexlab{b}})Cohen-Addad, Larsen, Saulpic, and Schwiegelshohn]{cohenaddad2022coreset}
Vincent Cohen-Addad, Kasper~Green Larsen, David Saulpic, and Chris Schwiegelshohn.
\newblock Towards optimal lower bounds for k-median and k-means coresets.
\newblock \emph{arXiv preprint arXiv:2202.12793}, 2022{\natexlab{b}}.
\newblock URL \url{https://arxiv.org/abs/2202.12793}.

\bibitem[Ergun et~al.(2022)Ergun, Feng, Silwal, Woodruff, and Zhou]{learning_augmented_kmeans}
Jon~C. Ergun, Zhili Feng, Sandeep Silwal, David~P. Woodruff, and Samson Zhou.
\newblock Learning-augmented {\textdollar}k{\textdollar}-means clustering.
\newblock 2022.

\bibitem[Fahad et~al.(2014)Fahad, Alshatri, Tari, Alamri, Khalil, Zomaya, Foufou, and Bouras]{graph_clustering_review}
Adil Fahad, Najlaa Alshatri, Zahir Tari, Abdullah Alamri, Ibrahim Khalil, Albert~Y Zomaya, Sebti Foufou, and Abdelaziz Bouras.
\newblock A survey of clustering algorithms for big data: Taxonomy and empirical analysis.
\newblock \emph{IEEE Transactions on Emerging Topics in Computing}, 2\penalty0 (3):\penalty0 267--279, 2014.
\newblock \doi{10.1109/TETC.2014.2330519}.

\bibitem[Floyd(1962)]{floyd1962algorithm}
Robert~W. Floyd.
\newblock Algorithm 97: Shortest path.
\newblock \emph{Communications of the ACM}, 5\penalty0 (6):\penalty0 345, 1962.
\newblock \doi{10.1145/367766.368168}.

\bibitem[Forgy(1965)]{forgy1965cluster}
Edward~W. Forgy.
\newblock Cluster analysis of multivariate data: Efficiency versus interpretability of classifications.
\newblock In \emph{Biometrics}, volume~21, pages 768--769. International Biometric Society, 1965.

\bibitem[Gamlath et~al.(2022)Gamlath, Makarychev, Makarychev, and Mazumdar]{gamlath2022noisylabels}
Basilios Gamlath, Konstantin Makarychev, Yury Makarychev, and Arya Mazumdar.
\newblock Approximate cluster recovery from noisy labels.
\newblock In \emph{Proceedings of the 35th Annual Conference on Learning Theory (COLT)}, pages 2456--2494, 2022.

\bibitem[Hamerly(2010)]{hamerly2010making}
Greg Hamerly.
\newblock Making k-means even faster.
\newblock \emph{Proceedings of the 2010 SIAM International Conference on Data Mining}, pages 130--140, 2010.
\newblock \doi{10.1137/1.9781611972801.12}.

\bibitem[Hartigan and Wong(1979)]{hartigan1979algorithm}
John~A. Hartigan and Manchek~A. Wong.
\newblock Algorithm as 136: A k-means clustering algorithm.
\newblock \emph{Journal of the Royal Statistical Society. Series C (Applied Statistics)}, 28\penalty0 (1):\penalty0 100--108, 1979.
\newblock \doi{10.2307/2346830}.

\bibitem[Inaba et~al.(1994)Inaba, Katoh, and Imai]{Inaba1994}
Mary Inaba, Naoki Katoh, and Hiroshi Imai.
\newblock Applications of weighted voronoi diagrams and randomization to variance-based k-clustering (extended abstract).
\newblock In \emph{Proceedings of the Tenth Annual Symposium on Computational Geometry}, pages 332--339, 1994.

\bibitem[Kriegel et~al.(2009)Kriegel, Kröger, and Zimek]{euclidean_limitations}
Hans-Peter Kriegel, Peer Kröger, and Arthur Zimek.
\newblock Clustering high-dimensional data.
\newblock \emph{ACM Transactions on Knowledge Discovery from Data (TKDD)}, 3\penalty0 (1):\penalty0 1--35, 2009.

\bibitem[Lee et~al.(2017)Lee, Schmidt, and Wright]{lee2017improved}
Euiwoong Lee, Melanie Schmidt, and John Wright.
\newblock Improved and simplified inapproximability for k-means.
\newblock \emph{Information Processing Letters}, 120:\penalty0 40--43, 2017.
\newblock ISSN 0020-0190.
\newblock \doi{https://doi.org/10.1016/j.ipl.2016.11.009}.
\newblock URL \url{https://www.sciencedirect.com/science/article/pii/S0020019016301739}.

\bibitem[Li and Svensson(2012)]{li2012pseudoapproximation}
Shi Li and Ola Svensson.
\newblock Approximating $k$-median via pseudo-approximation.
\newblock \emph{arXiv preprint arXiv:1211.0243}, 2012.
\newblock URL \url{https://arxiv.org/abs/1211.0243}.

\bibitem[Lloyd(1982)]{lloyd1982least}
Stuart Lloyd.
\newblock Least squares quantization in pcm.
\newblock \emph{IEEE Transactions on Information Theory}, 28\penalty0 (2):\penalty0 129--137, 1982.
\newblock \doi{10.1109/TIT.1982.1056489}.

\bibitem[Lucic et~al.(2016)Lucic, Bachem, and Krause]{lucic2016coresets}
Mario Lucic, Olivier Bachem, and Andreas Krause.
\newblock Strong coresets for hard and soft $k$-means clustering with applications to mixture models.
\newblock In \emph{Advances in Neural Information Processing Systems (NeurIPS)}, pages 504--512, 2016.
\newblock URL \url{https://arxiv.org/abs/1609.07148}.

\bibitem[MacQueen(1967)]{macqueen1967some}
James MacQueen.
\newblock Some methods for classification and analysis of multivariate observations.
\newblock In \emph{Proceedings of the Fifth Berkeley Symposium on Mathematical Statistics and Probability}, volume~1, pages 281--297. University of California Press, 1967.

\bibitem[Nguyen et~al.(2023)Nguyen, Chaturvedi, and Nguyen]{theoretical_guarantees}
Thy~Dinh Nguyen, Anamay Chaturvedi, and Huy~L. Nguyen.
\newblock Improved learning-augmented algorithms for k-means and k-medians clustering.
\newblock \emph{International Conference on Learning Representations (ICLR)}, 2023.

\bibitem[Sculley(2010)]{sculley2010webscale}
David Sculley.
\newblock Web-scale k-means clustering.
\newblock In \emph{Proceedings of the 19th International Conference on World Wide Web (WWW)}, pages 1177--1178. ACM, 2010.
\newblock \doi{10.1145/1772690.1772862}.

\bibitem[Silwal(2021)]{learning_augmented_algorithms}
Sandeep Silwal.
\newblock Learning-augmented algorithms.
\newblock \emph{Massachusetts Institute of Technology}, 2021.

\bibitem[Xu and Wunsch(2005)]{clustering_survey}
Rui Xu and Donald Wunsch.
\newblock A comprehensive survey of clustering algorithms.
\newblock \emph{Annals of Data Science}, 2005.

\bibitem[Yang et~al.(2022)Yang, Liu, Liang, Zhou, Liu, and Zhu]{graph_distances}
Xihong Yang, Yue Liu, Ke~Liang, Sihang Zhou, Xinwang Liu, and En~Zhu.
\newblock Attribute graph clustering via learnable augmentation.
\newblock \emph{arXiv preprint arXiv:2212.03559}, 2022.
\newblock URL \url{https://arxiv.org/abs/2212.03559}.

\bibitem[Yin et~al.(2024)Yin, Aryani, Petrie, Nambissan, Astudillo, and Cao]{yin2024rapid}
Hui Yin, Amir Aryani, Stephen Petrie, Aishwarya Nambissan, Aland Astudillo, and Shengyuan Cao.
\newblock A rapid review of clustering algorithms.
\newblock \emph{arXiv preprint arXiv:2401.07389}, 2024.
\newblock URL \url{https://arxiv.org/abs/2401.07389}.

\end{thebibliography}

\appendix

\end{document}